\documentclass[letterpaper]{article}
\usepackage{aaai}

\usepackage{times}
\usepackage{helvet}
\usepackage{courier}
\usepackage[noend,ruled,linesnumbered]{algorithm2e}
\usepackage{amsthm}
\usepackage{amsfonts}
\usepackage{graphics}
\usepackage{graphicx}
\usepackage{epsfig}
\usepackage{amsmath}

\newtheorem{mythem}{Theorem}

\newcommand{\probh}{{\bf Pr}}
\newcommand{\calDP}{\mathcal{\widetilde{D}}}
\newcommand{\candDP}{\langle\!\langle \calDP \rangle\!\rangle}
\newcommand{\calD}{\mathcal D}
\newcommand{\PreP}{\widetilde{Pre}}
\newcommand{\AddP}{\widetilde{Add}}
\newcommand{\DelP}{\widetilde{Del}}
\newcommand{\PreR}{\overline{Pre}}
\newcommand{\AddR}{\overline{Add}}
\newcommand{\DelR}{\overline{Del}}
\newcommand{\calP}{\mathcal{P}}
\newcommand{\calPP}{\mathcal{\widetilde{P}}}

\newcommand{\calO}{\mathcal{O}}

\nocopyright

\begin{document}
\title{Synthesizing Robust Plans under Incomplete Domain Models}
\author{
%Paper ID: 859
Tuan A. Nguyen {$^*$}
\and Subbarao Kambhampati {$^*$}
\and Minh B. Do {$^\dagger$} \\
{\small {* Dept. of Computer Science \& Engineering, Arizona State University.
Email: {\tt \{natuan,rao\}@asu.edu }}}\\
{$\dagger$} {\small Embedded Reasoning Area, Palo Alto Research Center.
Email: {\tt minh.do@parc.com}}
% \small{(For an updated version, please see: http://rakaposhi.eas.asu.edu/robust-plans.pdf)}
}
\maketitle

\begin{abstract}
\begin{quote}

  Most current planners assume complete domain models and focus on
  generating correct plans. Unfortunately, domain modeling is a
  laborious and error-prone task. While domain experts cannot guarantee
  completeness, often they are able to circumscribe the incompleteness
  of the model by providing annotations as to which parts of the domain
  model may be incomplete. In such cases, the goal should be to generate
  plans that are robust with respect to any known incompleteness of the
  domain. In this paper, we first introduce annotations expressing the
  knowledge of the domain incompleteness, and formalize the notion of
  plan robustness with respect to an incomplete domain model. We then
  propose an approach to compiling the problem of finding robust plans
  to the conformant probabilistic planning problem. We present
  experimental results with Probabilistic-FF, a state-of-the-art
  planner, showing the promise of our approach.

%   Most current planners assume complete domain models and focus on
%   generating correct plans. Unfortunately, domain modeling is a 
% laborious and error-prone task. While domain experts cannot
%   guarantee completeness, often they are able to circumscribe the
%   incompleteness of the model by providing annotations as to which parts
%   of the domain model may be incomplete. In such cases, the goal should
%   be to generate plans that are robust with respect to any known
%   incompleteness of the domain. In this paper, we formalize the notion
%   of plan robustness with respect to an incomplete domain model, and propose
%   an approach to compiling the problem of finding robust plans to the
%   conformant probabilistic planning problem. We present experimental
%   results with Probabilistic-FF, a state-of-the-art planner, showing the
%   promise of our approach.

% illustrate cases where the
%   approach works well. We end with a discussion of scalability issues.

\end{quote}
\end{abstract}

%\mvp

%----------
\section{Introduction}

In the past several years, significant strides have been made in scaling
up plan synthesis techniques. We now have technology to routinely
generate plans with hundreds of actions. All this work, however, makes a
crucial assumption--that a complete model of the domain is specified in
advance. While there are domains where knowledge-engineering such
detailed models is necessary and feasible (e.g., mission planning
domains in NASA and factory-floor planning), it is increasingly
recognized
(c.f. \cite{hoffmann2010sap,rao07}) that there are also many
scenarios 
%(e.g. Business Process Management \cite{hoffmann2010sap})
where insistence on correct and complete models renders the current
planning technology unusable. What we need to handle such cases is a
planning technology that can get by with partially specified domain
models, and yet generate plans that are ``robust'' in the sense that
they are likely to execute successfully in the real world.

This paper addresses the problem of formalizing the notion of plan
robustness with respect to an incomplete domain model, and connects the
problem of generating a robust plan under such model to \emph{conformant probabilistic
  planning}~\cite{kushmerick1995algorithm,hyafil2003conformant,bryce2006sequential,prob-ff}.
Following Garland \& Lesh~\shortcite{garland2002plan}, we shall assume
that although the domain modelers cannot provide complete models, often
they are able to provide annotations on the partial model circumscribing
the places where it is incomplete. In our framework, these annotations
consist of allowing actions to have \emph{possible} preconditions and
effects (in addition to the standard necessary preconditions and
effects). 

As an example, consider a variation of the \emph{Gripper} domain, a
well-known planning benchmark domain. The robot has one gripper that can be
used to pick up balls, which are of two types light and heavy, from one
room and move them to another room. The modeler suspects that the gripper
may have an internal problem, but this cannot be confirmed until the
robot actually executes the plan. If it actually has the problem, the
execution of the \emph{pick-up} action succeeds only with balls that are
\emph{not} heavy, but if it has no problem, it can always pickup all
types of balls. The modeler can express this partial knowledge about the
domain by annotating the action with a statement representing the
possible precondition that balls should be light.

Incomplete domain models with such possible
preconditions/effects implicitly define an exponential set of
complete domain models, with the semantics that the real
domain model is guaranteed to be one of these. The robustness of a
plan can now be formalized in terms of the cumulative probability mass of the complete
domain models under which it succeeds. We propose an approach that compiles the
problem of finding robust plans into the conformant probabilistic
planning problem. We present
experimental results showing scenarios where the approach works well,
and also discuss aspects of the compilation that cause scalability issues.
% that make it fail, providing
% an evidence of the limitation on exploiting current planning techniques
% to synthesize robust plans.

\section{Related Work}
%\noindent {\bf Related work:} 
Although there has been some work on reducing the ``faults'' in plan execution
(e.g.  the work on \emph{k-fault}
plans for non-deterministic planning~\cite{jensen2004fault}), it is
based in the context of stochastic/non-deterministic actions rather
than incompletely specified ones.
%applicable for our problem.
The semantics of the possible preconditions/effects in our incomplete
domain models differ fundamentally from non-deterministic and
stochastic effects. Executing different instances of the same pick-up
action in the \emph{Gripper} example above would either all fail or
all succeed, since there is no uncertainty but the information is
unknown at the time the model is
built. 
%This distinction also leads to the difference between
% ``robust plans'' and good stochastic plans. If the robot has only one hand
% and the modeler does not know if it's good, then the most robust plan is
% to use that hand, but the robustness of the plan is measured by 50\%
% success-rate, {\em no matter how many instances of that action we
%   concatenate into the plan}. 
In contrast, if the pick-up action's
effects are stochastic, then trying the same picking action multiple
times increases the chances of success. 

Garland \& Lesh ~\shortcite{garland2002plan} share the same objective
with us on generating robust plans under incomplete domain models.
However, their notion of robustness, which  is defined in terms of four
different types of risks, only has tenuous heuristic connections
with the likelihood of successful execution of plans. Robertson \& Bryce
~\shortcite{robertson09} focuses on the plan generation in Garland \&
Lesh model, but their approach still relies on the same
unsatisfactory formulation of robustness. The work by Fox et al
(\citeyear{fox05}) also explores robustness of plans, but their focus is
on temporal plans under unforeseen execution-time variations rather
than on incompletely specified domains. Our work can also be categorized as one
particular instance of the general model-lite planning problem, as
defined in \cite{rao07}, in which the author points out a large class of
applications where handling incomplete models  is unavoidable due to the difficulty in getting a
complete model.

\section{Problem Formulation}

%Extending the formalism introduced by Garland \& Lesh
%~\shortcite{garland2002plan}, 
We define an \emph{incomplete domain
  model} $\calDP$ as $\calDP = \langle F, A \rangle$, where $F=\{p_1, p_2,
..., p_n\}$ is a set of \emph{propositions}, $A$ is a set of
\emph{actions} that might be incompletely specified. We denote
$\mathbf{T}$ and $\mathbf{F}$ as the \emph{true} and \emph{false} truth
values of propositions. A \emph{state} $s \subseteq F$ is a set of
propositions. In addition to proposition sets that are known as its
preconditions $Pre(a) \subseteq F$, add effects $Add(a) \subseteq F$ and
delete effects $Del(a) \subseteq F$, each action $a \in A$ also
contains:

\begin{itemize}
\item Possible precondition set $\PreP(a) \subseteq F$ contains propositions
  that action $a$ \emph{might} need as its preconditions.

\item Possible add (delete) effect set $\AddP(a) \subseteq F$
  ($\DelP(a) \subseteq F$) contains propositions
  that the action $a$ \emph{might} add (delete, respectively) after its
  execution.

\end{itemize}

In addition, each possible precondition, add and delete effect $p$ of
the action $a$ is associated with a weight $w^{pre}_a(p)$,
$w^{add}_a(p)$ and $w^{del}_a(p)$ ($0 < w^{pre}_a(p), w^{add}_a(p),
w^{del}_a(p) < 1$) representing the domain modeler's assessment of
the likelihood that $p$ will actually be \emph{realized} as a
precondition, add and delete effect of $a$ (respectively) during plan
execution. Possible preconditions and effects whose likelihood of
realization is not given are assumed to have weights of
$\frac{1}{2}$.

% We denote $PreR^{*}(a)$, $\AddR^{*}(a)$ and $\DelR^{*}(a)$ as the sets of
% realized (and unknown) preconditions and effects of each action $a \in
% A$ ($PreR^{*}(a) \subseteq \PreP(a)$, $\AddR^{*}(a) \subseteq
% \AddP(a)$, $\DelR^{*}(a) \subseteq \DelP(a)$), and $\calD^*$ as the \emph{real} domain model in which action $a$
% has its complete sets of preconditions and effects $Pre(a) \cup
% PreR^{*}(a)$, $Add(a) \cup \AddR^{*}(a)$ and $Del(a) \cup \DelR^{*}(a)$.

%Why not actually talk about elaboration of domain models??
%If we want minimally consistent, we need to talk of completions
Given an incomplete domain model $\calDP$, we define its {\em
completion set} $\candDP$ as the set of {\em complete} domain models
whose actions have all the necessary preconditions, adds and deletes,
and a  {\em subset} of the possible preconditions, possible adds and
possible deletes. 
Since any subset of $\PreP(a)$, $\AddP(a)$ and $\DelP(a)$ can be realized as preconditions
and effects of action $a$, there are exponentially large number of
possible \emph{complete} domain models $\calD_i \in \candDP  = \{\calD_1, \calD_2,
..., \calD_{2^K}\}$, where $K = \sum_{a \in A}
(|\PreP(a)| + |\AddP(a)| + |\DelP(a)|)$. For each complete model $\calD_i$,
we denote the corresponding sets of realized
preconditions and effects for each action $a$ as $\PreR_i(a)$, $\AddR_i(a)$
and $\DelR_i(a)$; equivalently, its complete sets of preconditions and
effects are $Pre(a) \cup \PreR_i(a)$, $Add(a) \cup \AddR_i(a)$ and
$Del(a) \cup \DelR_i(a)$.

The projection of a sequence of actions $\pi$ from an initial state $I$ according to
an incomplete domain model $\calDP$ is defined in terms of the projections 
of $\pi$ from $I$ according to  each complete domain  model  $\calD_i \in \candDP$:

\begin{equation}
 \gamma( \pi, I , \calDP) = \bigcup_{\calD_i \in \candDP} \gamma( \pi
, I, \calD_i ) 
\label{projection}
\end{equation}

\noindent
where the projection over complete models is defined in the usual STRIPS  way, with one
important difference. The result of applying an action $a$ in a state
$s$ where the preconditions of $a$ are not satisfied is taken to be
$s$ (rather than as an undefined state).\footnote{We shall see that this change
is necessary so that we can talk about increasing the robustness of a
plan by adding additional actions.}

%{\bf I suggest we get rid of the following para in favor of what I
%  added above}
%Given that the \emph{real} domain model, which we call
%$\calD^{*}$, can be any one in $\candDP$, the resulting of applying an
%action $a$ in a state $s$ is defined with respect to each ``candidate''
%complete model for $\calD^{*}$. Specifically, we define the resulting state with
%respect to $\calD_i \in \candDP$ as $\gamma_{i}(s,a) = (s \setminus
%(Del(a) \cup \DelR_i(a))) \cup Add(a) \cup \AddR_i(a) $ if $Pre(a) \cup
%\PreR_i(a) \subseteq s$, and $\gamma_{i}(s,a) = s$ otherwise. We note
%that the semantics of action execution in our problem is defined
%differently from that in STRIPS model in a sense that executing an
%action $a$ with unsatisfied preconditions in a state $s$ (i.e., $Pre(a)
%\cup \PreR_i(a) \not\subseteq s$) does \emph{not} result in an
%``invalid'' state, but makes the state unchanged.\footnote{In many
%  real-world applications, this requirement is necessary so that
%  additional action can be inserted into a plan to increase chance of
%  achieving goals.}

A \emph{planning problem with incomplete domain } is $\calPP = \langle \calDP ,I,G \rangle$ where
$I \subseteq F$ is the set of propositions that are true  in the
\emph{initial state}, and $G$ is the set of \emph{goal propositions}.
An action sequence $\pi$ is considered a {\bf  valid plan} for $\calPP$ if
$\pi$ solves the problem in at least one completion of
$\candDP$. Specifically, 
$\exists_{\calD_i \in \candDP} \gamma(\pi , I , \calD_i) \models G $. 

%% We denote $a_I, a_G \not\in A $ as two dummy actions representing the 
%% initial and goal state such that $Pre(a_I) = \emptyset$, $Add(a_I) = I$,
%% $Pre(a_G) = G$, $Add(a_G) = \{\top\}$ (where $\top \not\in F$ denotes a
%% dummy proposition representing goal achievement).
%A \emph{plan} for the problem $\calP$ is a sequence of actions
%$\pi=(a_1, ..., a_n)$ such that there exists a complete domain
%$\calD_i \in \candDP$ in which all goal propositions hold in the state
%$\gamma_{i}(...\gamma_{i}(\gamma_{i}(I,a_1),a_2)...,a_{n})$.
%This notion of plan ensures that plans which can actually solves $\calP$
%in the real domain $\calD^{*}$, which is even though unknown, are all included in the solution space.
%% in which $a_1$ is applicable in $I$ and every 
%% action $a_i$ is applicable in the state $s_i =
%% \gamma(...\gamma(\gamma(I,a_1),a_2)...,a_{i-1})$ $(1 < i \leq n)$. 

% In the presence of $\PreP(a)$, $\AddP(a)$ and $\DelP(a)$, the execution of a
% plan $\pi$ might not reach a goal state (i.e. the plan \emph{fails})
% when some possible precondition or effect of an action $a$ is realized,
% i.e. winds up holding in the \emph{real} domain model $\calD^*$. Note
% that we ignore $\PreP(a)$ and $\DelP(a)$, but not $\AddP(a)$, in the
% transition function defined above to ensure completeness.  Thus, if there is a plan
% reaching a goal state in at least one \emph{candidate} complete
% model for $\calD^*$, then it is not excluded.

\medskip
\noindent
{\bf Modeling Issues in Annotating Incompleteness}:
From the modeling point of view, the possible precondition and effect
sets can be modeled at either the grounded action or action schema level
(and thus applicable to all grounded actions sharing the same action
schema). % Going back to the \emph{Gripper} domain mentioned earlier, the
% possible internal problem may be on either the left hand or both hands,
% depending on whether the incompleteness is specified at the grounded or
% schema level of the \emph{pick-up} action.
% From a practical point of view, however, incompleteness annotations at
% ground level hugely increase the burden on the domain modeler. To offer
% a flexible way in modeling the domain incompleteness, we allow
% annotations that are restricted to specific grounding of the
% corresponding action schema. In particular, given variables $x_i$ with
% domains $X_i$, one can indicate that
% $p(x_{i_1},...,x_{i_k})$ is a possible precondition/effect of an action
% schema $a(x_1,...,x_n)$ when some variables $x_{j_1}, ..., x_{j_l}$ have
% values $c_{1} \in X_{j_1},..., c_{l} \in X_{j_l}$. Those possible
% preconditions/effects can be specified with the annotation $p(x_{i_1},...,x_{i_k}) \,\, :when
% \,\, (x_{j_1}=c_{1}, ..., x_{j_l} = c_{l})$ for action schema
% $a(x_1,...,x_n)$. This syntax subsumes both the annotations at the ground
% level when $l=n$, and at the schema level if $l=0$ (or the $:when$ construct
% is not specified). 
From a practical point of view, however, incompleteness annotations at
ground level hugely increase the burden on the domain modeler. To offer
a flexible way in modeling the domain incompleteness, we allow
annotations that are restricted to either specific variables or value
assignment to variables of an action schema. In particular:

\begin{itemize}

\item \textit{Restriction on value assignment to variables}: Given
  variables $x_i$ with domains $X_i$, one can indicate that
  $p(x_{i_1},...,x_{i_k})$ is a possible precondition/effect of an
  action schema $a(x_1,...,x_n)$ when some variables $x_{j_1}, ...,
  x_{j_l}$ have values $c_{j_1} \in X_{j_1},..., c_{j_l} \in X_{j_l}$
  ($\{i_1,...,i_k\},\{j_1,...,j_l\} \subseteq \{1,...,n\}$). Those
  possible preconditions/effects can be specified with the annotation
  $p(x_{i_1},...,x_{i_k}) \,\, :when \,\, (x_{j_1}=c_{1} \wedge ... \wedge x_{j_l} =
  c_{l})$ for the action schema $a(x_1,...,x_n)$. More generally, we
  allow the domain writer to express a constraint
  $C$ on the variables $x_{j_1}, ..., x_{j_l}$ in the $:when$
  construct. The annotation $p(x_{i_1},...,x_{i_k}) \,\, :when \,\,(C)$
  means that $p(c_{i_1},...,c_{i_k})$ is a
  possible precondition/effect of an instantiated action
  $a(c_1,...,c_n)$ ($c_i \in X_i$) if and only if the assignment
  $(x_{j_1}:=c_{j_1}, ..., x_{j_l}:=c_{j_l})$ satisfies the constraint $C$. This syntax subsumes both
  the annotations at the ground level when $l=n$, and at the schema
  level if $l=0$ (or the $:when$ construct is not specified).

% \item \textit{Restriction on value assignment to variables}: Given
%   variables $x_i$ with domains $X_i$, one can indicate that
%   $p(x_{i_1},...,x_{i_k})$ is a possible precondition/effect of an
%   action schema $a(x_1,...,x_n)$ when some variables $x_{j_1}, ...,
%   x_{j_l}$ have values $c_{1} \in X_{j_1},..., c_{l} \in X_{j_l}$. Those
%   possible preconditions/effects can be specified with the annotation
%   $p(x_{i_1},...,x_{i_k}) \,\, :when \,\, (x_{j_1}=c_{1}, ..., x_{j_l} =
%   c_{l})$ for the action schema $a(x_1,...,x_n)$. This syntax subsumes both
%   the annotations at the ground level when $l=n$, and at the schema
%   level if $l=0$ (or the $:when$ construct is not specified).

\item \textit{Restriction on variables}: Instead of constraints on
  explicit values of variables, we also allow the possible
  preconditions/effects $p(x_{i_1},...,x_{i_k})$ of an action schema
  $a(x_1,...,x_n)$ to be dependent on some specific variables $x_{j_1}
  ,..., x_{j_l}$ \emph{without any knowledge of their restricted
    values}. This annotation essentially requires less amount of
  knowledge of the domain incompleteness from the domain writer.
  Semantically, the possible precondition/effect $p(x_{i_1},...,x_{i_k})
  \,\, :depends \,\, (x_{j_1},..., x_{j_l})$ of an action schema
  $a(x_1,...,x_n)$ means that (1) there is at least one instantiated
  action $a(c_1,...,c_n)$ ($c_i \in X_i$) having
  $p(c_{i_1},...,c_{i_k})$ as its precondition, and (2) for
  any two assignments $(x_1:=c_1,...,x_n:=c_n),
  (x_1:=c'_1,...,x_n:=c'_n)$ such that $c_{j_t} = c'_{j_t}$ ($1 \leq t
  \leq l$), either both $p(c_{i_1},...,c_{i_k})$ and
  $p(c'_{i_1},...,c'_{i_k})$ are preconditions of the corresponding
  actions, or they are not. Similar to the $:when$ above, the $:depend$
  construct also subsumes the annotations at the ground level when
  $l=n$, and at the schema level if $l=0$ (or the $:depend$ field is not
  specified).

\end{itemize}

% Another flexibility is offered in terms of an optional ``depends''
% field for the possible precondition/effect annotations with the
% syntax: $(P u_1 \cdots u_n) \,\, :depends \,\, v_1..v_n$ with the
% semantics that the action will have the possible precondition/effect
% $((P u_1 \cdots u_n) $ whose realization depends on the action schema
% variables $v_1...v_n$. If the $:depends$ field is empty, it means that
% every instance of that schema has the same exact possible
% precondition/effect. If the depends field contains all the variables
% of the action schema, then, each action instantiation can have an
% independently realized possible preconditions/effects (thus, in
% effect, allowing ground level annotations).

%In using $\PreP$, $\AddP$ and $\DelP$ annotations, we are using an
%assumption, which we call \emph{uncorrelated incompleteness}: the
%incomplete preconditions and effects are all assumed to be independent
%of each other. Our current representation thus does not allow a domain writer to
%state, for example, that a particular action $a$ will have the possible add
%effect $e$ only when it has the possible precondition $p$. It is not
%too hard to extend the representation to support such correlated
%incompleteness; in particular our compilation strategy can be extended
%in a fairly straightforward way to handle it. 
%%While we
%%cannot completely rule out a domain modeler capable of making
%%annotations about such correlated sources of incompleteness, we assume
%%that this is less likely.
%
Another interesting modeling issue is the correlation among the
possible preconditions and effects across actions. In particular, the
domain writer might want to say that two actions (or action schemas)
will have specific possible preconditions and effects in tandem. For
example, we might say that the second action will have a particular
possible precondition whenever the first one has a particular possible
effect. We note that annotations at the lifted level introduce
correlations among possible preconditions and effects at the ground
level.

Although our notion of plan robustness and approach to generating robust
plans (see below) can be adapted to allow such flexible
annotations and correlated incompleteness, for ease of exposition we
limit our discussion to \emph{uncorrelated} possible precondition and effect
annotations specified at the \emph{schema} level (i.e. without using the
$:when$ and $:depend$ constructs).

\section{A Robustness Measure for Plans}

%Given that a plan $\pi = (a_1, ..., a_n)$ for
%$\calP = \langle \calD,I,G \rangle$ with an incomplete domain model
%$\calDP$, as defined earlier, may either succeed or fail to reach a goal
%state, we aim to search for a plan that most likely succeeds with
%respect to the real domain model $\calD^{*}$; in other words, we seek
%for a maximally ``robust'' plan. We now formalize the
%notion of plan robustness.

Given an incomplete domain planning problem $\calPP = \langle \calDP
,I,G \rangle$, a valid plan (by our definition above) need only to succeed 
in at least one completion of $\calDP$. Given that $\candDP$ can be exponentially large
in terms of possible preconditions and effects, validity 
is too weak to guarantee on the quality of the plan. What we need is a
notion that $\pi$ succeeds in most of the highly likely completions
of $\calDP$. We do this in terms of a robustness measure.

% We define the \emph{robustness} measure $R(\pi)$ for a plan $\pi$ as the
% probability that it succeeds in achieving $G$ with respect to
% $\calD^{*}$. More formally, let $K = \sum_{a \in A}
% (|\PreP(a)| + |\AddP(a)| + |\DelP(a)|)$, $\candDP = \{\calD_1, \calD_2,
% ..., \calD_{2^K}\}$ be the set of the candidate models of $\calD^{*}$
% and $h:\ \candDP \rightarrow [0,1]$ ($\sum_{1 \le i \le 2^K}
% \probh(\calD_i) = 1$) be the distribution function representing the modeler's
% estimate of the probability that a given model in $\candDP$ is
% actually $\calD^*$, the robustness value of a plan $\pi$ is then defined
% as follows:

The robustness of a plan $\pi$ for the problem $\calPP$
%: \langle \calDP ,I,G \rangle$ 
is defined as the cumulative probability mass of the
completions of $\calDP$ under which $\pi$ succeeds (in achieving the
goals).  More formally, %
let 
%h: \candDP \rightarrow [0,1]$ ($\sum_{1
%  \le i \le 2^K}
$\probh(\calD_i)$ be the probability distribution 
representing the modeler's estimate of the probability that a given
model in $\candDP$ is the real model of the world (such that
$\sum_{\calD_i \in \candDP} \probh(\calD_i) = 1$). The robustness of
$\pi$ is defined as follows:
%is actually $\calD^*$, the robustness value of a
%plan $\pi$ is then defined as follows:
%
%\vspace{-.09in}
\begin{equation}
\label{eqn:robust-def}
R(\pi, \calPP: \langle \calDP
,I,G \rangle ) \stackrel{def}{\equiv} \sum_{\calD_i \in \candDP , \gamma(\pi , I , \calD_i) \models G }\probh(\calD_i)
\end{equation}

%\noindent where $\prod \subseteq \candDP$ is the set of candidate models 
%with which $\pi$ succeeds in achieving $G$.

It is easy to see that if
$R(\pi, \calPP ) > 0$, then $\pi$ is a valid plan for $\calPP$. 

Note that given the uncorrelated incompleteness assumption, the
probability $\probh(\calD_i)$ for a model $\calD_i \in \candDP$ can be
computed as the product of the weights $w^{pre}_a(p)$, $w^{add}_a(p)$,
and $w^{del}_a(p)$ for all $a \in A$ and its possible
preconditions/effects $p$ if $p$ \emph{is} realized in the model
%as its precondition, add and delete effect in $\calD_i$ 
(or the product of
their ``complement'' $1-w^{pre}_a(p)$, $1-w^{add}_a(p)$, and
$1-w^{del}_a(p)$ if $p$ is \emph{not} realized). % There is a scenario, which we
% call \emph{non-deterministic incompleteness}, when the domain writer
% does not have any quantitative measure of likelihood as to whether each
% (independent) possible precondition/effect will be realized or not. In
% this case, we will handle non-deterministic uncertainty as ``uniform''
% distribution over models.\footnote{as is typically done when
%   distributional information is not available--since uniform
%   distribution has the highest entropy and thus makes least amount of
%   assumptions.}
% The robustness of $\pi$ can then be computed as follows:

% \begin{equation}
% \label{eqn: robust-uniform-def}
% R(\pi) = \frac{|\prod|}{2^K}
% \end{equation} 

% Begin Figure
\begin{figure}[t]
\centering
\epsfig{file=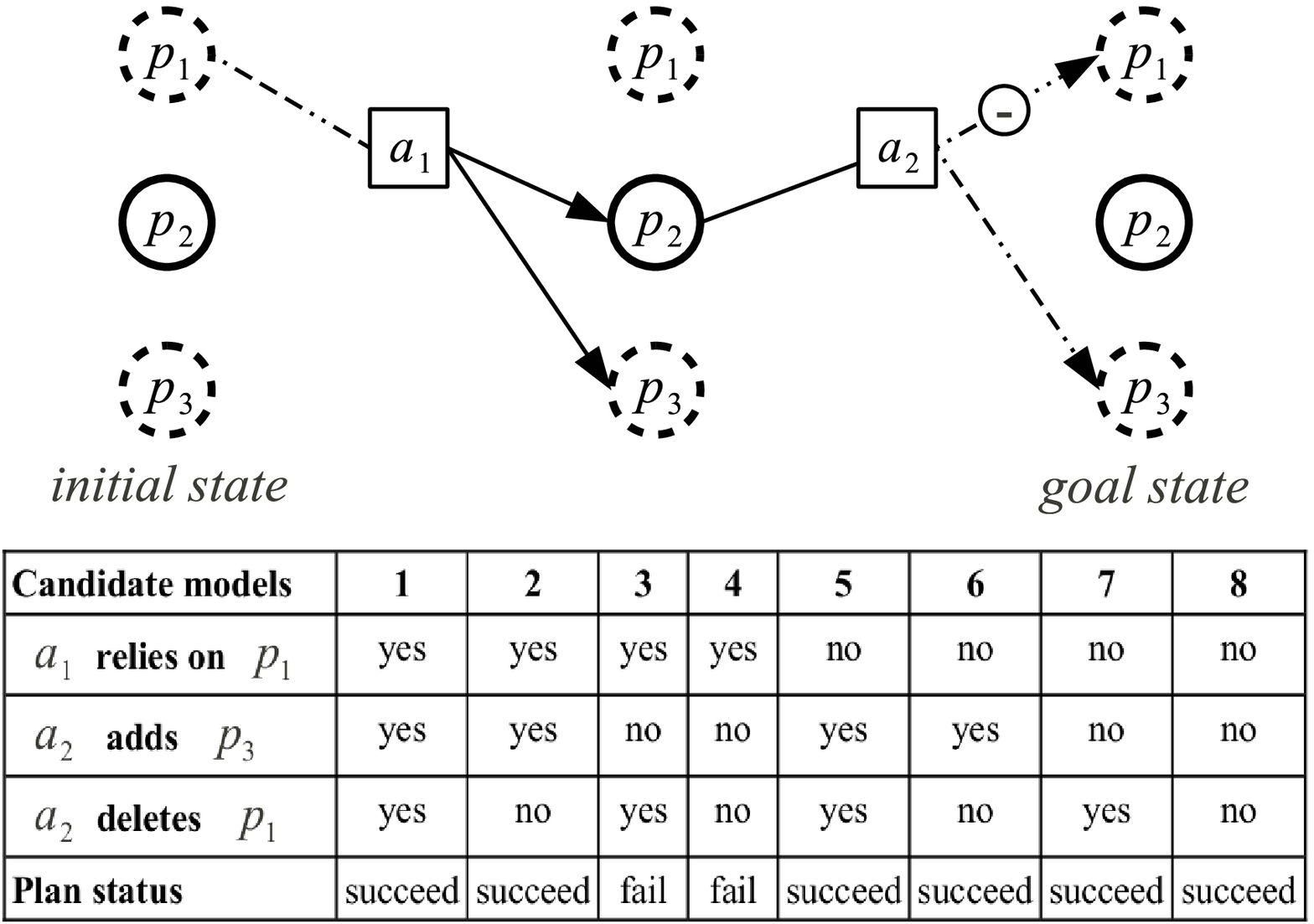,width=3.3in}
\caption{An example of different complete domain models,
  and the corresponding plan status. Circles with solid and dash
  boundary are propositions that are known to be
  $\mathbf{T}$ and may be $\mathbf{F}$ (respectively) when the plan executes. (See text.)
}
\vspace{-.05in}
\label{fig:realization}
\end{figure}
% End Figure

\medskip
\noindent {\bf Example:} Figure~\ref{fig:realization} shows an example
with an incomplete  domain model $\calDP = \langle F, A \rangle$ with
$F=\{p_1,p_2,p_3\}$ and $A=\{a_1,a_2\}$ and a solution plan
$\pi=(a_1,a_2)$ for the problem $\calPP= \langle \calDP,I=\{p_2\},
G=\{p_3\} \rangle$. The incomplete model is:
$Pre(a_1) = \emptyset$, $\PreP(a_1) = \{p_1\}$, $Add(a_1) = \{p_2,p_3\}$,
$\AddP(a_1)=\emptyset$, $Del(a_1)=\emptyset$, $\DelP(a_1)=\emptyset$;
$Pre(a_2) = \{p_2\}$, $\PreP(a_2) = \emptyset$, $Add(a_2) = \emptyset$,
$\AddP(a_2)=\{p_3\}$, $Del(a_2)= \emptyset$, $\DelP(a_2)=\{p_1\}$. Given that the
total number of possible preconditions and effects is 3, the total
number of completions ($| \candDP | $)  is $2^3 = 8$, for
each of which the plan $\pi$ may succeed or fail to achieve $G$, as shown in
the table. % The candidate model 6, for instance, corresponds to the
% scenario where the first action $a_1$ does not depend on $p_3$ but it
% deletes $p_1$. Even though $a_2$ could execute, it does not have $p_1$
% as an additive effect, and the plan fails to achieve $p_1$ as a goal. In
% summary, there are $5$ of $8$ candidate models where $\pi$ fails and $3$
% candidate models of $\calD^{*}$ for which $\pi$ succeeds. 
The robustness value of the plan is $R(\pi) = \frac{3}{4}$ if $\probh(\calD_i)$ is the
uniform distribution. However, if the domain writer thinks that $p_1$ is
very likely to be a precondition of $a_1$ and provides
$w^{pre}_{a_1}(p_1) = 0.9$, the robustness of $\pi$ decreases to $R(\pi)
= 2 \times (0.9 \times 0.5 \times 0.5) + 4 \times (0.1 \times 0.5 \times
0.5) = 0.55$ (as intutively, the
last four models with which $\pi$ succeeds are very unlikely to be the
real one). Note that under the STRIPS model where action failure causes
plan failure, the plan $\pi$ would considered failing to achieve $G$ in the
first two complete models, since $a_2$ is prevented from execution.
% For the rest of this paper, we
% first focus on the situation with the assumption of
% non-deterministic incompleteness, but some of our techniques can be
% adapted for the more general case.

\subsection{A Spectrum of Robust Planning Problems}
Given this set up, we can now talk about a spectrum of problems
related to planning under incomplete domain models:

\begin{description}

\item[Robustness Assessment (RA):] Given a plan $\pi$ for the problem
  $\calPP$, assess the robustness of $\pi$.

\item[Maximally Robust Plan Generation (RG$^*$):] Given a problem
  $\calPP$, generate the maximally robust plan $\pi^*$.

\item[Generating Plan with Desired Level of Robustness (RG$^\rho$):]
  Given a problem $\calPP$ and a robustness threshold $\rho$ ($0 < \rho
  \leq 1$), generate a plan $\pi$ with
  robustness greater than or equal to $\rho$.

\item[Cost-sensitive Robust Plan Generation (RG$^*_c$):] Given a problem
  $\calPP$ and a cost bound $c$, generate a plan $\pi$ of maximal
  robustness subject to cost bound $c$ (where the cost of a plan $\pi$
  is defined as the cumulative costs of the actions in $\pi$).

\item[Incremental Robustification (RI$_c$):] Given a plan $\pi$ for the
  problem $\calPP$, improve the robustness of $\pi$, subject to a cost
  budget $c$.

\end{description}

% Given that the incomplete domain model may have exponential number of 
% completions in the worst case, RA, assessing the robustness of a plan is
% computationally hard. In fact, we can show that model counting problem
% can be reduced to robustness assessment, making  RA a \#P-Complete problem. 
% In practice, RA can be tackled by  compiling it into a weighted SAT
% model counting problem.
% \fromRao{Tuan: Make sure this is still correct. In particular, does
%   the method handle the fact that different models correspond to
%   different weights? I am also not wedded to keeping this in the final
% paper if the readability suffers..}
% In this approach, we setup a SAT encoding $E$ such that there is a
% one-to-one map between each model of $E$ with a candidate domain model
% $\calD_i \in \candDP$.
% %Thus, counting the number of models of $E$ should gives us
% %$M_{\pi}$. 
% To encode the executability of $\pi$, we will use
% constraints representing the \emph{causal-proof} (c.f.
% \citeauthor{mali1999utility} \citeyear{mali1999utility}) of the
% correctness of $\pi$. In essence, the SAT constraints enforce that
% whenever an action $a_i \in \pi$ needs a precondition $p \in Pre(a_i)$
% or $p \in \PreP(a_i)$, then (i) $p$ is established at some level $j \leq
% i$ and (ii) there is no action that deletes $p$ between $j$ and $i$. In
% level $i$, we means the $i^{th}$ state progressed from the initial state
% $I$ by applying the action sequence $a_1,....a_i$.  

%\fromRao{Tuan: you can add more details here if you
%  want--especially if we have space..}  

The problem of assessing robustness of plans, RA, can be tackled by compiling
it into a weighted model-counting problem. For plan synthesis problems,
we can talk about either generating a maximally
robust plan, RG$^*$, or finding a plan with a robustness value above the given
threshold, RG$^\rho$. A related issue is that of the interaction between plan
cost and robustness. Often, increasing robustness involves using
additional (or costlier) actions to support the desired goals, and
thus  comes at the
expense of increased plan cost. We can also talk about
cost-constrained robust plan generation problem RG$^*_c$. Finally, in
practice, we are often interested in increasing the robustness of a
given plan (either during iterative search, or during mixed-initiative
planning). We thus also have the incremental variant RI$_c$. 
% \fromRao{note to rao: should put in the reason for precondition model here..}

In this paper, we will focus on RG$^\rho$, the problem of
synthesizing plan with at least a robustness value of $\rho$. % In the next section we
% shall show how this problem can be compiled into a conformant
% probabilistic planning problem. 

% \fromRao{RG variants should have the complexity of the probabilistic
%   conformant planning..} 

%\fromRao{should we also describe the cost consideration? maximally
%  robust plan vs. maximally robust plan below a cost bound? It is
%  probably better to define it but choose to ignore costs for now,
%  than to not bring it up..}
%\fromRao{Should we say something about robustness assessment too? We
%  should at least point out that robustness assessment is NP-hard! }

\section{Compilation to Conformant Probabilistic Planning}

%Given our formulation of plan robustness, we can pose two problems: how
%to assess the robustness of a given plan, and how to generate a
%maximally robust plan. The first can be handled efficiently by compiling
%to a weighted-model counting problem. Our focus is on the second---robust plan synthesis
%problem. 
In this section, we will show that the problem of generating plan with
at least $\rho$ robustness, RG$^\rho$, can be compiled into an equivalent conformant probabilistic
planning problem. The most robust plan can then be
found with a sequence of increasing threshold
values.

% RG$^*$  can then be
% tackled with a sequence of increasing threshold
% values. 

% \footnote{Although the threshold $\rho$ is not defined as part of
%   our problem, we introduce it here to fit the requirement of the
%   Probabilistic-FF planner that we are using.}

\subsection{Conformant Probabilistic Planning}

Following the formalism in \cite{prob-ff}, a domain in conformant
probabilistic planning (CPP) is a tuple $\calD'= \langle F', A'
\rangle$, where $F'$ and $A'$ are the sets of propositions and
probabilistic actions, respectively. A belief state $b: 2^{F'}
\rightarrow [0,1]$ is a distribution of states $s \subseteq F'$ (we
denote $s \in b$ if $b(s) > 0$). Each action $a' \in A'$ is specified by
a set of preconditions $Pre(a') \subseteq F'$ and conditional effects
$E(a')$.  For each $e=(cons(e),\calO(e)) \in E(a')$, $cons(e) \subseteq
F'$ is the condition set and $\calO(e)$ determines the set of outcomes
$\varepsilon=(Pr(\varepsilon),add(\varepsilon),del(\varepsilon))$ that
will add and delete proposition sets $add(\varepsilon)$,
$del(\varepsilon)$ into and from the resulting state with the
probability $Pr(\varepsilon)$ ($0 \leq Pr(\varepsilon) \leq 1$ ,
$\sum_{\varepsilon \in \calO(e)} Pr(\varepsilon) = 1$). All condition sets
of the effects in $E(a')$ are assumed to be mutually exclusive and
exhaustive. The action $a'$ is applicable in a belief state $b$ if
$Pre(a') \subseteq s$ for all $s \in b$, and the probability of a state
$s'$ in the resulting belief state is
% determined by
$b_{a'}(s') = \sum_{s \supseteq Pre(a')} b(s) \sum_{\varepsilon \in
  \calO'(e)} Pr(\varepsilon)$, where $e \in E(a')$ is the conditional
effect such that $cons(e) \subseteq s$, and $\calO'(e) \subseteq
\calO(e)$ is the set of outcomes $\varepsilon$ such that $s' = s \cup
add(\varepsilon) \setminus del(\varepsilon)$.

Given the domain $\calD'$, a problem $\calP'$ is a quadruple $\calP' =
\langle \calD',b_I,G',\rho' \rangle$, where $b_I$ is an initial belief
state, $G'$ is a set of goal propositions and $\rho'$ is the acceptable
goal satisfaction probability. A sequence of actions $\pi'=(a_1',...,
a_n')$ is a solution plan for $\calP'$ if $a_i'$ is applicable in the
belief state $b_i$ (assuming $b_1 \equiv b_I$), which results in
$b_{i+1}$ ($1 \leq i \leq n$), and it achieves all goal propositions
with at least $\rho'$ probability.

\subsection{Compilation} 

% We now propose an approach to finding a plan with at least
% $\rho$ robustness value that translates both the domain $\calD$ and problem
% $\calP$ into $\calD'$ and $\calP'$ as specified in the CPP setting,
% which can then be given as input to a conformant probabilistic
% planner.

% Given the definition of projection of a plan over an
% incomplete domain model (Equation~\ref{projection}), it is easy to see
% the parallel between planning with incomplete domain models and
% conformant probabilistic planning. In particular, we note that in both
% cases, the result of the projection is a ``belief state;'' the
% difference is that the uncertainty in conformant planning is in the
% initial state, while for us it is in the domain model. Our strategy
% then is to compile the uncertainty from the domain model into the
% initial state.
% Tuan noted: in CPP the uncertainty is on both initial state and domain
% model.

Given an incomplete domain model $\calDP=\langle F,A \rangle$ and a
planning problem $\calPP=\langle \calDP,I,G \rangle$, we now describe a
compilation that translates the problem of synthesizing a solution plan
$\pi$ for $\calPP$ such that $R(\pi,\calPP) \geq \rho$ to a CPP problem
$\calP'$. At a high level, the realization of possible preconditions $p \in
\PreP(a)$ and effects $q \in \AddP(a)$, $r \in \DelP(a)$ of an action $a
\in A$ can be understood as being determined by the truth values of
\emph{hidden} propositions $p_a^{pre}$, $q_a^{add}$ and $r_a^{del}$ 
that are certain (i.e. unchanged in any world state) but
unknown. Specifically, the applicability of the action in a state $s \subseteq F$
depends on possible preconditions $p$ that are realized (i.e. $p_a^{pre}
= \mathbf{T}$), and their truth values in $s$.  Similarly, the values of
$q$ and $r$ are affected by $a$ in the resulting state only if they are
realized as add and delete effects of the action (i.e., $q_a^{add} =
\mathbf{T}$, $r_a^{del} = \mathbf{T}$).  There are totally
$2^{|\PreP(a)|+|\AddP(a)|+|\DelP(a)|}$ realizations of the action $a$, and
all of them should be considered simultaneously in checking the
applicability of the action and in defining corresponding resulting
states.

With those observations, we use multiple conditional effects to compile
away incomplete knowledge on preconditions and effects of the action $a$.
Each conditional effect corresponds to one realization of the action,
and can be fired only if $p = \mathbf{T}$ whenever $p_a^{pre} =
\mathbf{T}$, and adding (removing) an effect $q$ ($r$) into (from) the
resulting state depending on the values of $q_a^{add}$ ($r_a^{del}$,
respectively) in the realization.

While the partial knowledge can be removed, the hidden propositions
introduce uncertainty into the initial state, and therefore making it a
\emph{belief} state. Since the action $a$ may be applicable in some
but rarely all states of a belief state, \emph{certain}
preconditions $Pre(a)$ should be modeled as conditions of all
conditional effects. We are now ready to formally specify the resulting
domain $\calD'$ and problem $\calP'$.

For each action $a \in A$, we introduce new propositions
$p_a^{pre}$, $q_a^{add}$, $r_a^{del}$ and their negations $np_a^{pre}$, $nq_a^{add}$, $nr_a^{del}$ for each $p \in \PreP(a)$, $q
\in \AddP(a)$ and $r \in \DelP(a)$ to determine whether they are realized
as preconditions and effects of $a$ in the real domain.\footnote{These
  propositions are introduced once, and re-used for all actions sharing
  the same schema with $a$.} Let $F_{new}$ be the set
of those new propositions, then $F' = F \cup F_{new}$ is the proposition
set of $\calD'$.

Each action $a' \in A'$ is made from one action $a \in A$ such that
$Pre(a') = \emptyset$, and $E(a')$ consists of
$2^{|\PreP(a)|+|\AddP(a)|+|\DelP(a)|}$ conditional effects $e$. For each
conditional effect $e$:

% The
% condition set $cons(e)$ and \emph{single} outcome $\varepsilon$ of each
% conditional effect $e$ are defined as follows:

% \begin{itemize}
%   \item $cons(e) = Pre(a) \cup \PreR(a) \cup \{ p_a^{pre} | p \in \PreR(a)
%     \} \cup \{np_a^{pre} | p \in \PreP(a) \setminus \PreR(a)  \}
%               \cup \{ q_a^{add} | q \in \AddR(a) \} \cup \{ nq_a^{add} |
%               q \in \AddP(a) \setminus \AddR(a) \} \cup \{ r_a^{del} | r \in
%               \DelR(a) \} \cup \{ nr_a^{del} | r \in \DelP(a) \setminus \DelR(a) \}$,

%   \item $add(\varepsilon) = Add(a) \cup \AddR(a)$, $del(\varepsilon) =
%     Del(a) \cup \DelR(a)$, and $Pr(\varepsilon) = 1$,

% \end{itemize}

\begin{itemize}
  \item $cons(e)$ is the union of the following sets:

    \begin{itemize}
      \item the certain preconditions $Pre(a)$,
      \item the set of possible preconditions of $a$ that are realized,
        and hidden propositions representing their
        realization: $\PreR(a) \cup \{ p_a^{pre} | p \in
        \PreR(a) \} \cup \{np_a^{pre} | p \in \PreP(a) \setminus
        \PreR(a) \}$,
      \item the set of hidden propositions corresponding to the realization
        of possible add (delete) effects of $a$: $\{ q_a^{add} | q \in \AddR(a) \} \cup \{ nq_a^{add} |
              q \in \AddP(a) \setminus \AddR(a) \}$ ($\{ r_a^{del} | r \in
              \DelR(a) \} \cup \{ nr_a^{del} | r \in \DelP(a) \setminus
              \DelR(a) \}$, respectively);

    \end{itemize}

  \item the \emph{single} outcome $\varepsilon$ of $e$ is defined as $add(\varepsilon) = Add(a) \cup \AddR(a)$, $del(\varepsilon) =
    Del(a) \cup \DelR(a)$, and $Pr(\varepsilon) = 1$,

\end{itemize}

\noindent
where $\PreR(a) \subseteq \PreP(a)$, $\AddR(a) \subseteq \AddP(a)$ and
$\DelR(a) \subseteq \DelP(a)$ represent the sets of realized preconditions
and effects of the action. In other words, we create a
conditional effect for each subset of the union of the possible
precondition and effect sets of the action $a$. Note that the inclusion
of new propositions derived from $\PreR(a)$, $\AddR(a)$, $\DelR(a)$ and their
``complement'' sets $\PreP(a) \setminus \PreR(a)$, $\AddP(a) \setminus
\AddR(a)$, $\DelP(a) \setminus \DelR(a)$ makes all condition sets of the
action $a'$ mutually exclusive. As for other cases (including those in which
some precondition in $Pre(a)$ is excluded), the action has no effect on
the resulting state, they can be ignored. The condition sets, therefore,
are also exhaustive.

The initial belief state $b_I$ consists of $2^{|F_{new}|}$ states $s'
\subseteq F'$ such that $p \in s'$ iff $p \in I$ ($\forall p \in F$),
each represents a complete domain model $\calD_i \in \candDP$ and
with the probability $\probh(\calD_i)$. The goal is $G' = G$, and the acceptable
goal satisfaction probability is $\rho' = \rho$.

\begin{mythem}
  Given a plan $\pi = (a_1,..., a_n)$ for the problem $\calPP$, and $\pi' =
  (a_1', ..., a_n')$ where $a_k'$ is the compiled version of $a_k$ ($1
  \leq k \leq n$) in $\calP'$. Then $R(\pi,\calPP) \geq \rho$ iff $\pi'$
  achieves all goals with at least $\rho$ probability in $\calP'$.
\end{mythem}

\begin{proof}[Proof (sketch)]
  According to the compilation, there is one-to-one mapping between each
  complete model $\calD_i \in \candDP$ in $\calPP$ and a (complete)
  state $s_{i0}' \in b_I$ in $\calP'$. Moreover, if $\calD_i$ has a
  probability of $\probh(\calD_i)$ to be the real model, then $s_{i0}'$
  also has a probability of $\probh(\calD_i)$ in the belief state $b_I$
  of $\calP'$.

  Given our projection over complete model $\calD_i$, executing $\pi$
  from the state $I$ with respect to $\calD_i$ results in a
  sequence of complete state $(s_{i1}, ..., s_{i(n+1)})$. On the other
  hand, executing $\pi'$ from $\{s_{i0}'\}$ in $\calP'$ results in a
  sequence of belief states $(\{s_{i1}'\}, ..., \{s_{i(n+1)}'\})$. With
  the note that $p \in s_{i0}'$ iff $p \in I$ ($\forall p \in F$), by
  induction it can be shown that $p \in s_{ij}'$ iff $p \in s_{ij}$
  ($\forall j \in \{1,...,n+1\}, p \in F$). Therefore, $s_{i(n+1)}
  \models G$ iff $s_{i(n+1)}' \models G = G'$.

  Since all actions $a_i'$ are deterministic and $s_{i0}'$ has a
  probability of $\probh(\calD_i)$ in the belief state $b_I$ of
  $\calP'$, the probability that $\pi'$ achieves $G'$ is
  $\sum_{s_{i(n+1)}' \models G} \probh(\calD_i)$, which is equal to
  $R(\pi,\calPP)$ as defined in Equation~\ref{eqn:robust-def}. This
  proves the theorem.
\end{proof}

% EXAMPLE
\begin{figure}[t]
\centering
\epsfig{file=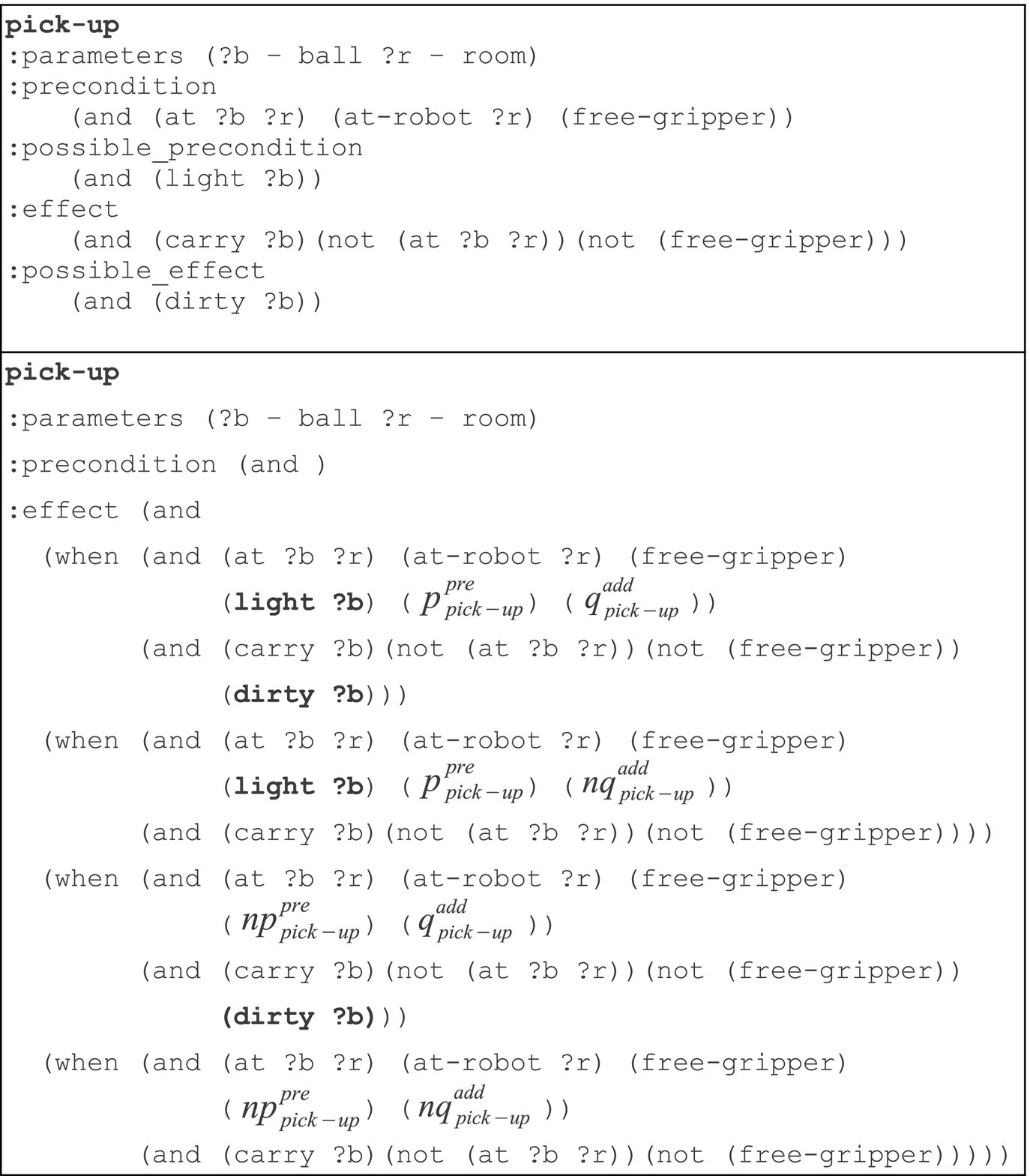,width=3.3in}
\caption{An example of compiling the action \emph{pick-up} in an incomplete
  domain model (top) into CPP domain (bottom). The hidden propositions
  $p_{pick-up}^{pre}$, $q_{pick-up}^{add}$ and their negations can be interpreted as
  whether the action requires light balls and makes balls dirty. Newly introduced and
  relevant propositions are marked in bold.}
\label{fig:compilation-example}
\end{figure}

\medskip
\noindent
{\bf Example:} Consider the action \emph{pick-up(?b - ball,?r - room)} in
the Gripper domain as described above. In addition to the possible
precondition \emph{(light ?b)} on the weight of the ball \emph{?b}, we also assume that since the
modeler is unsure if the gripper
has been cleaned or not, she models it with a possible add effect
\emph{(dirty ?b)} indicating
that the action might make the ball
dirty. Figure~\ref{fig:compilation-example} shows both the original and
the compiled specification of the action.

%\section{Pruning Useless Actions}

%\vspace{-0.1in}
\section{Experimental Results}

We tested the compilation with Probabilistic-FF (PFF), a
state-of-the-art planner, on a range of domains in the International Planning
Competition.% \footnote{More information on the domains can be found at:
%   http://ipc.icaps-conference.org/.}
We first discuss the results on the
variants of the Logistics and Satellite domains, where domain
incompleteness is deliberately modeled on the preconditions and effects
of actions (respectively). Our purpose here is
to observe how generated plans are robustified to satisfy a given
robustness threshold, and how the amount of incompleteness in
the domains affects the plan generation phase. We then describe the
second experimental setting in which we randomly introduce incompleteness
into IPC domains, and discuss the feasibility of our approach in this setting.\footnote{The experiments were
  conducted using an Intel Core2 Duo 3.16GHz machine with 4Gb of RAM,
  and the time limit is 15 minutes.}

\medskip
\noindent
{\bf Domains with deliberate incompleteness}

\noindent\textit{Logistics}: In this domain, each of the two cities $C_1$
and $C_2$ has an airport and a downtown area. The transportation between
the two distant cities can only be done by two airplanes $A_1$ and
$A_2$. In the downtown area of $C_i$ ($i \in \{1,2\}$), there are three \emph{heavy}
containers $P_{i1}, ..., P_{i3}$ that can be moved to the airport by a
truck $T_i$. Loading those containers onto the truck in the city $C_i$,
however, requires moving a team of $m$ robots $R_{i1}, ..., R_{im}$ ($m
\geq 1$), initially located in the airport, to the downtown area. The source of
incompleteness in this domain comes from the assumption that each pair
of robots $R_{1j}$ and $R_{2j}$ ($1 \leq j \leq m$) are made by the same manufacturer
% \fromRao{the significance of same manufacturer not very clear to me;
%   makes me think of correlated incompleteness which I am sure is not
%   what you want to talk about}
$M_{j}$, both therefore might fail to load a \emph{heavy} container.\footnote{The
  \emph{uncorrelated incompleteness} assumption applies for possible
  preconditions of action schemas specified for different manufacturers. It should
not be confused here that robots $R_{1j}$ and $R_{2j}$ of the same
manufacturer $M_j$ can independently have fault.} The
actions loading containers onto trucks using robots made by a particular
manufacturer (e.g., the action schema \emph{load-truck-with-robots-of-M1} using
robots of manufacturer $M_1$), therefore, have a \emph{possible precondition}
requiring that containers should not be heavy. To simplify discussion
(see below), we assume that robots of
different manufacturers may fail to load heavy containers, though
independently, with the same probability of $0.7$. The goal is to
transport all three containers in the city $C_1$ to $C_2$, and vice
versa. For this domain, a plan to ship a container to another city
involves a step of loading it onto the truck, which can be done by a
robot (after moving it from the airport to the downtown). Plans can be
made more robust by using additional robots of \emph{different}
manufacturer after moving them into the downtown areas, with the cost of
increasing plan length.

\noindent\textit{Satellite}: In this domain, there are two satellites $S_1$
and $S_2$ orbiting the planet Earth, on each of which there are $m$
instruments $L_{i1}, ..., L_{im}$ ($i \in \{1,2\}$, $m \geq 1$) used to
take images of interested modes at some direction in the space. For each $j \in
\{1,...,m\}$, the lenses of instruments $L_{ij}$'s were made from a type
of material $M_j$, which might have an error affecting the quality of
images that they take. If the material $M_j$ actually has error, all instruments $L_{ij}$'s
produce mangled images. The knowledge of this incompleteness is modeled
as a \emph{possible add effect} of the action taking images using
instruments made from $M_j$ (for instance, the
action schema \emph{take-image-with-instruments-M1} using instruments of type $M_1$)
with a probability of $p_j$, asserting that images taken might be in a
bad condition. A typical plan to take an image using an instrument, e.g.
$L_{14}$ of type $M_4$ on the satellite $S_1$, is first to switch on
$L_{14}$, turning the satellite $S_1$ to a ground direction from which
$L_{14}$ can be calibrated, and then taking image. Plans can be made
more robust by using additional instruments, which might be on a different
satellite, but should be of \emph{different} type of materials and can also take an
image of the interested mode at the same direction.

%---
\begin{table} {\scriptsize
\begin{center}
\begin{tabular}{| c || c | c | c | c | c |} 
\hline

$\rho$ & $m=1$ & $m=2$ & $m=3$ & $m=4$ & $m=5$ \\

\hline \hline

$0.1$ & $32/10.9$ & $36/26.2$ & $40/57.8$  & $44/121.8$ & $48/245.6$ \\

\hline

$0.2$ & $32/10.9$ & $36/25.9$  & $40/57.8$ & $44/121.8$ & $48/245.6$ \\

\hline

$0.3$ & $32/10.9$ & $36/26.2$ & $40/57.7$ & $44/122.2$ & $48/245.6$ \\

\hline

$0.4$ & $\bot$ & $42/42.1$ & $50/107.9$  & $58/252.8$ & $66/551.4$ \\

\hline

$0.5$ & $\bot$ & $42/42.0$ & $50/107.9$  & $58/253.1$  & $66/551.1$ \\

\hline

$0.6$ & $\bot$ & $\bot$  & $50/108.2$  & $58/252.8$ & $66/551.1$ \\

\hline

$0.7$ & $\bot$ & $\bot$ & $\bot$  & $58/253.1$ & $66/551.6$\\

\hline

$0.8$ & $\bot$ & $\bot$ & $\bot$ & $\bot$ & $66/550.9$\\

\hline

$0.9$ & $\bot$ & $\bot$ & $\bot$ & $\bot$ & $\bot$ \\

\hline
\end{tabular}

\caption{The results of generating robust plans in Logistics domain (see
  text).}
\vspace{-.1in}
\label{table:logistics}
\end{center} }

\end{table}
%---

\begin{table} {\scriptsize
\begin{center}
\begin{tabular}{| c || c | c | c | c | c |} 
\hline

$\rho$ & $m=1$ & $m=2$ & $m=3$ & $m=4$ & $m=5$ \\

\hline \hline

$0.1$ & $10/0.1$ & $10/0.1$ & $10/0.2$  & $10/0.2$ & $10/0.2$ \\

\hline

$0.2$ & $10/0.1$ & $10/0.1$  & $10/0.1$ & $10/0.2$ & $10/0.2$ \\

\hline

$0.3$ & $\bot$ & $10/0.1$ & $10/0.1$ & $10/0.2$ & $10/0.2$ \\

\hline

$0.4$ & $\bot$ & $37/17.7$ & $37/25.1$  & $10/0.2$ & $10/0.3$ \\

\hline

$0.5$ & $\bot$ & $\bot$ & $37/25.5$  & $37/79.2$  & $37/199.2$ \\

\hline

$0.6$ & $\bot$ & $\bot$  & $53/216.7$  & $37/94.1$ & $37/216.7$ \\

\hline

$0.7$ & $\bot$ & $\bot$ & $\bot$  & $53/462.0$ & -- \\

\hline

$0.8$ & $\bot$ & $\bot$ & $\bot$ & $\bot$ & -- \\

\hline

$0.9$ & $\bot$ & $\bot$ & $\bot$ & $\bot$ & $\bot$ \\

\hline
\end{tabular}

\caption{The results of generating robust plans in Satellite domain (see
  text).}
\vspace{-.15in}
\label{table:satellite}
\end{center} }

\end{table}
%---

Table~\ref{table:logistics} and \ref{table:satellite} shows respectively
the results in the Logistics and Satellite domains with $\rho \in \{0.1,
0.2, ..., 0.9\}$ and $m = \{1,2,...,5\}$. The number of complete domain
models in the two domains is $2^{m}$. For Satellite domain, the
probabilities $p_j$'s range from $0.25$, $0.3$,... to $0.45$ when $m$
increases from $1$, $2$, ... to $5$. For each specific value of $\rho$
and $m$, we report $l/t$ where $l$ is the length of plan and $t$ is the
running time (in seconds). Cases in which no plan is found within the
time limit are denoted by ``--'', and those where it is provable
that no plan with the desired robustness exists are denoted by ``$\bot$''.

\textit{Observations on fixed value of $m$}: In both domains, for a fixed value of $m$ we observe that the
solution plans tend to be longer with higher robustness threshold
$\rho$, and the time to synthesize plans is also larger. For instance,
in Logistics with $m=5$, the plan returned has $48$ actions if
$\rho=0.3$, whereas $66$-length plan is needed if $\rho$ increases to $0.4$. Since loading
containers using the same robot multiple times does not increase the
chance of success, more robots of different manufacturers need to move into the downtown area for
loading containers, which causes an increase in plan length. In the Satellite
domain with $m=3$, similarly, the returned plan has $37$ actions when $\rho=0.5$, but
requires $53$ actions if $\rho = 0.6$---more actions need to calibrate an
instrument of different material types in order to increase the chance of having a
good image of interested mode at the same direction.

Since the cost of actions is currently ignored in the compilation
approach, we also observe that more than the needed number of actions have been
used in many solution plans. In the Logistics domain,
specifically, it is easy to see that the
probability of successfully loading a container onto a truck using
robots of $k$ ($1 \leq k \leq m$) different manufacturers is $(1 -
{0.7}^{k})$. As an example, however, robots of all five manufacturers are used in a plan
when $\rho=0.4$, whereas using those of three manufacturers is enough.

\textit{Observations on fixed value of $\rho$}: In both domains, we
observe that the maximal robustness value of plans that can be returned
increases with higher number of manufacturers (though the higher the
value of $m$ is, the higher number of complete models is). For instance,
when $m=2$ there is not any plan returned with at least $\rho = 0.6$ in
the Logistics domain, and with $\rho=0.4$ in the Satellite domain. Intuitively,
more robots of different manufacturers offer higher probability of successfully
loading a container in the Logistics domain (and similarly for instruments of
different materials in the Satellite domain).

Finally, it may take longer time to synthesize plans with the same
length when $m$ is higher---in other words, the increasing amount of incompleteness
of the domain makes the plan generation phase harder. As an example, in the
Satellite domain, with $\rho=0.6$ it takes $216.7$ seconds to synthesize a $37$-length
plan when there are $m=5$ possible add effects at the schema level of the domain, whereas
the search time is only $94.1$ seconds when $m=4$. With $\rho=0.7$, no
plan is found within the time limit when $m=5$, although a plan with
robustness of $0.7075$ exists in the solution space. It is the increase
of the branching factors and the time spent on satisfiability test
and weighted model-counting used inside the planner % \footnote{These operations are used in
%   PFF to compute resulting belief states during the search. We ignore the
%   detailed discussion due to space issue.}
that affect the search efficiency.

\medskip
\noindent {\bf Domains with random incompleteness} 

\noindent
We built a program to generate an incomplete domain model from a
deterministic one by introducing $M$ new propositions into each domain
(all are initially $\mathbf{T}$). Some of those new propositions were
randomly added into the sets of \emph{possible} preconditions/effects of
actions. Some of them were also randomly made \emph{certain} add/delete effects
of actions. With this strategy, each solution plan in an original
deterministic domain is also a \emph{valid plan}, as defined
earlier, in the corresponding incomplete domain. Our experiments with
the Depots, Driverlog, Satellite and ZenoTravel domains indicate that
because the annotations are random, there are often fewer opportunities
for the PFF planner to increase the robustness of a plan prefix during
the search. This makes it hard to generate plans with a desired level of
robustness under given time constraint.

In summary, our experiments on the two settings above suggest that
the compilation approach based on the PFF planner would be a reasonable
method for generating robust plans in domains and problems where there
are chances for robustifying existing action sequences in the search space.

\section{Conclusion and Future Work}

In this paper, we motivated the need for synthesizing robust plans under
incomplete domain models. We introduced annotations for expressing
domain incompleteness, formalized the notion of plan robustness, and showed
an approach to compile the problem of generating robust plans into
conformant probabilistic planning. We presented empirical results
showing the promise of our approach. For future work, we are
developing a planning approach that directly takes the incompleteness
annotations into account during the search, and compare it with our current
compilation method. We also plan to consider the problem of robustifying
a given  plan subject to a provided cost bound.

%AAAI-11 submission version
% In this paper, we motivated the need for synthesizing robust plans under
% incomplete domain models. We presented a framework for representing
% incomplete models, formalized the notion of plan robustness, and showed
% an approach to compile the problem of generating robust plans into
% conformant probabilistic planning. We presented empirical results
% showing the promise of our approach. For future work, we are
% developing a planning approach that directly takes the incompleteness
% annotations into account during the search, and compare it with our current
% compilation method. We also plan to consider the problem of robustifying
% a given  plan subject to a provided cost bound.

% modifying the relaxed plan extraction
% routine in FF so it is greedily biased towards more robust heuristic
% completions of the current partial plan.

% \noindent {\bf Acknowledgement:} This research is supported in part by
% ONR grants N00014-09-1-0017, N00014-07-1-1049 and the NSF grant
% IIS-0905672.

\medskip
\noindent {\bf Acknowledgement:} This research is supported in part by
ONR grants N00014-09-1-0017 and N00014-07-1-1049, the NSF grant
IIS-0905672, and by DARPA and the U.S. Army Research Laboratory under
contract W911NF-11-C-0037. The content of the information does not
necessarily reflect the position or the policy of the Government, and no
official endorsement should be inferred. We thank William Cushing for
several helpful discussions.

%\vspace{-0.01in}
%\footnotesize{
\bibliography{bib.bib}
\bibliographystyle{aaai}
%}

\end{document}